\documentclass[12pt]{msml2021} 

\makeatletter
\let\Ginclude@graphics\@org@Ginclude@graphics 
\makeatother

\usepackage{bm}
\usepackage{color}

\newcommand{\RR}{\mathbb{R}}

\newcommand{\EE}{\mathbb{E}}

\newcommand{\bx}{\bm{x}}
\newcommand{\bfm}{\mathbf{m}}
\newcommand{\bv}{\mathbf{v}}

\newcommand{\bX}{\bm{X}}
\newcommand{\bu}{\bm{u}}

\newcommand{\br}{\bm{r}}

\newcommand{\bi}{\begin{enumerate}}
\newcommand{\ei}{\end{enumerate}}

\title[Dynamic Behavior of Adam]{A Qualitative Study of the Dynamic Behavior for Adaptive Gradient Algorithms}
\usepackage{times}

\msmlauthor{%
 \Name{Chao Ma} \Email{chaoma@stanford.edu}\\
 \addr Department of Mathematics, Stanford University
 \AND
 \Name{Lei Wu} \Email{leiwu@princeton.edu}\thanks{The first two authors contributed equally.}\\
 \addr Program in Applied and Computational Mathematics, Princeton University
 \AND
 \Name{Weinan E} \Email{weinan@math.princeton.edu}\\
 \addr Program in Applied and Computational Mathematics and Department of Mathematics, Princeton University
}

\begin{document}

\maketitle

\begin{abstract}%
The dynamic behavior of RMSprop and Adam algorithms is studied through a combination of careful numerical experiments and theoretical explanations. Three  types of qualitative features are observed  in the training loss curve: fast initial convergence,  oscillations, and large spikes in the late phase. The sign gradient descent (signGD) flow, which is the limit of Adam when taking the learning rate to $0$ while keeping the momentum parameters fixed, is used to explain the fast initial convergence. For the late phase of Adam, three different types of qualitative patterns are observed depending on the choice of the hyper-parameters: oscillations, spikes, and divergence. In particular, Adam converges much smoother and even faster when the values of the two momentum factors are close to each other. This observation is particularly important  for scientific computing tasks, for which the training process usually proceeds into the high precision regime.
\end{abstract}


\section{Introduction}
Adaptive gradient algorithms \citep{duchi2011adaptive, Tieleman2012, kingma2014adam}, in particular  RMSprop \citep{Tieleman2012} and Adam \citep{kingma2014adam}, have demonstrated superior performance in training modern machine learning models, e.g. deep neural networks.  Distinguished from the vanilla gradient descent (GD) or  stochastic gradient descent (SGD),  adaptive gradient algorithms use a coordinate-wise scaling of the update direction. The scaling factors are adaptively determined by using the history of past gradients \citep{duchi2011adaptive}, which makes the understanding and analysis of these algorithms much more challenging.

Recent theoretical efforts \citep{reddi2018convergence,zhou2018convergence,xie2020linear,li2019convergence, chen2018convergence} have  focused on establishing the convergence of adaptive gradient algorithms. However, these results are still unsatisfactory, since they cannot explain any of the particular features of these adaptive gradient algorithms. Moreover, all these results require  taking the limit that the learning rate $\eta_t$ goes to zero, e.g.   $\eta_t = 1/\sqrt{t}$.  However, in practice, one usually starts with a large learning rate and only decays the learning rate several times during the training process. For the most of iterations, the learning rate is fixed.  So it is interesting to see what happens when the learning rate is fixed. 
Figure~\ref{fig: loss} shows the dynamical behavior of full-batch Adam with a fixed learning rate for one classification problem and one regression problem, respectively. For both cases, one can see that the training curve does not behave monotonically even for this full batch setting. Large spikes keep repeatedly appearing in the late phase of the training. 

These spikes may not cause serious problems for typical computer vision and NLP tasks. For these applications, the data are either highly noisy or very limited, the learning is mainly dominated by the generalization gap (as shown in the left panel of Figure \ref{fig: loss}). Typically, a training loss ~$10^{-2}\sim 10^{-3}$ is sufficient to select a good model.
However, for many scientific computing tasks, such as fitting a target function and numerically solving PDEs, we are more interested in high-precision solutions. For these problems, the training data are relatively clean and easy to obtain, which reduces the risk of overfitting. Hence, lower training loss is always desired. The large spikes in the late phase make it difficult to pick a good stopping time, as shown in the right panel of Figure \ref{fig: loss}.

\begin{figure}[!h]
    \centering
    \includegraphics[width=0.4\columnwidth]{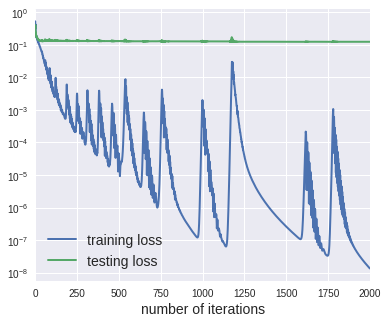}
    \includegraphics[width=0.4\columnwidth]{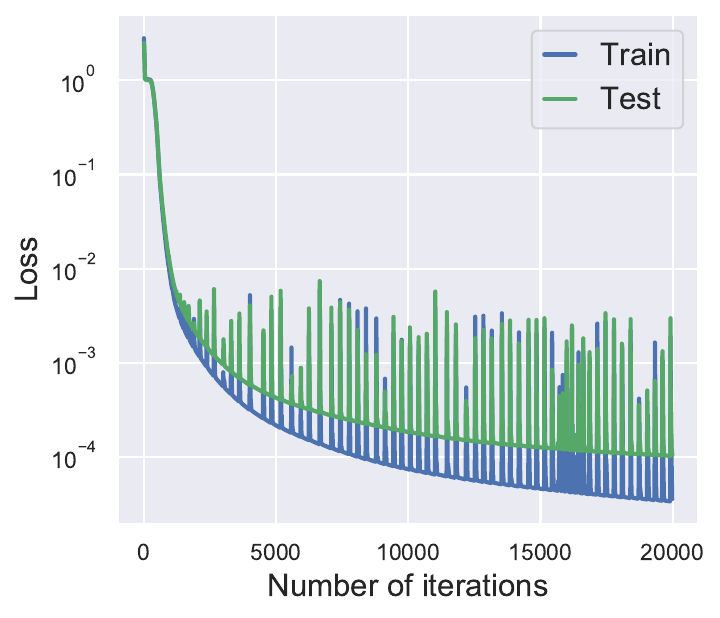}
    \vspace{-4mm}
    \caption{\small The training curves of full batch Adam. For both cases,  the learning rate is fixed to be $0.001$ and $(\beta_1, \beta_2)=(0.9, 0.999)$, the default values in PyTorch and TensorFlow.  \textbf{Left:} Classify CIFAR-10 dataset with a fully-connected neural networks. The network has $3$ hidden layers with widths 256-256-128. 
    $2$ classes are picked from CIFAR-10  with $1000$ images in each class. Square loss function is used. \textbf{Right:} Fit the target function $f^*(\bx)=\sum_{k=1}^5 \sin(2\pi x_k)$ for $\bx\in [0,1]^5$ with a fully connected network, whose architecture is 5-100-100-1. $2000$ points are uniformly sampled from $\text{Unif}([0,1]^5)$ to form the training set }.
    \label{fig: loss}
\end{figure}

In addition to the learning rate, adaptive gradient algorithms also use extra hyper-parameters such as the second-order momentum factor for RMSprop and the first and second order momentum factors for Adam. Default values of these hyper-parameters are provided in mainstream packages (e.g. $\beta_1$=0.9, $\beta_2$=0.999 for Adam in PyTorch and TensorFlow),  which are usually tuned on the classification problems with cross-entropy loss \citep{kingma2014adam}. However, they are not necessarily optimal for regression and scientific computing problems, where the quadratic loss is used, e.g., the default values can cause the large spikes as shown in Figure \ref{fig: loss}.
One objective of this paper is to carry out a comprehensive study of how the choice of these hyper-parameters affects the dynamical behavior.

\paragraph{Contributions}
In this paper, we provide well-designed experiments to demystify the dynamic behavior of adaptive gradient algorithms. Specifically, our contributions are summarized as follows.

\begin{enumerate}
    \item We identify three types of typical phenomena in the training process of these adaptive algorithms: initial fast convergence (sometimes even super-linear), small oscillations, and large spikes in the late phase.
    
    \item For RMSprop and Adam, if the learning rate decreases to zero while the momentum parameters are fixed, the algorithms converge to the signGD flow. For signGD flow, we prove the finite-time convergence for objective functions that satisfy the Polyak-Lojasiewicz (PL) \citep{polyak1963gradient} condition. These arguments together provide a partial explanation of the fast initial convergence of RMSprop and Adam, which could be the one of  reasons behind the popularity of these algorithms.
    
    \item We show that the large spikes are caused by some instabilities of the algorithm at stationary points. For RMSprop on simple objective functions, we explicitly write down the limiting oscillating solution. For Adam, we classify the behavior into three different patterns in
    the space of the two momentum factors: {the spike regime, the oscillation regime, and the divergence regime.} Empirical results show that training is most stable in the ``oscillation regime'',in particular when $\beta_1\approx\beta_2$.
\end{enumerate}

Throughout this paper, all the activation functions are $\sigma(t)=\max(0,t)$, unless explicitly specified.
The quadratic loss is used for all the experiments, including the classification problems.
Other experimental details are described in the caption of each figure.

To make the notations more consistent, from now on we use $\alpha$ to denote the second-order momentum factor in both Adam and RMSprop, and use $\beta$ to denote the first-order momentum in Adam. The conventional notations $\beta_1$ and $\beta_2$  for Adam will become $\beta$ and $\alpha$, respectively. For vectors $\bu$ and $\bv$, operations such as $\bu^2$, $\sqrt{\bu}$, $\bu/\bv$, and $|\bu|$ are 
understood  to be element-wise.  

\section{Preliminaries}\label{sec:2}
\subsection{Adaptive gradient algorithms}
Adaptive gradient algorithms are a family of optimization algorithms that use a coordinate-wise scaling of the update direction (gradient or gradient with momentum)  according to the history of gradients. Many adaptive algorithms can be cast to the following form~\citep{da2018general},

\begin{equation}\label{eqn: adap_general}
\begin{aligned}
    \bfm_{t+1} &= h_t\nabla f(\bx_t) + r_t \bfm_t \\
    \bv_{t+1}  &= p_t (\nabla f(\bx_t))^2 + q_t \bv_t\\
    \bx_{t+1}  &= \bx_t - \eta_t \frac{\bfm_{t+1}}{\sqrt{\bv_{t+1}}+\epsilon},
\end{aligned}
\end{equation}
with different choice of $h, r, p, q$. In~\eqref{eqn: adap_general}, $h, r, p, q$ are scalar functions of $t$. For example, Adagrad~\citep{duchi2011adaptive} is recovered when $h, p, q=1$ and $r=0$, and RMSprop corresponds to the case when $h=1$, $r=0$, $p=1-\alpha$ and $q=\alpha$ for some constant $\alpha\in(0,1)$. Viewed from the dynamics of $\bx_t$ alone, adaptive gradient algorithms usually have a ``memory effect'' due to the momentum terms. The strength of the memory depends on the momentum factors ($h_t, r_t, p_t, q_t$) and the learning rate $\eta_t$.
Because of their efficiency in training neural network models, these  algorithms are extensively used. We refer readers  to~\citep{ruder2016overview} for a more thorough review of existing adaptive algorithms.

In this paper, we focus on RMSprop and Adam --- the two algorithms that are most widely used by practitioners. The discrete update rules of these algorithms are
\begin{itemize}
\item {\bf RMSprop:}
\begin{equation} \label{eqn: rmsprop}
\begin{aligned}
  \bv_{t+1} &= \alpha \bv_{t} + (1-\alpha)(\nabla f(\bx_t))^2 \\
  \bx_{t+1} &= \bx_t - \eta\frac{\nabla f(\bx_t)}{\sqrt{\bv_{t+1}}+\epsilon}
\end{aligned}
\end{equation}

\item {\bf Adam:}
\begin{equation} \label{eqn: adam}
    \begin{aligned}
  \bv_{t+1} &= \alpha \bv_{t} + (1-\alpha) (\nabla f(\bx_t))^2\\
  \bfm_{t+1} &= \beta \bfm_{t} + (1-\beta) \nabla f(\bx_t)\\
  \bx_{t+1} &= \bx_t - \eta\frac{\bfm_{t+1}/(1-\beta^{t+1})}{\sqrt{\bv_{t+1}/(1-\alpha^{t+1})}+\epsilon}
\end{aligned}
\end{equation}

\end{itemize}
In~\eqref{eqn: rmsprop} and~\eqref{eqn: adam}, $\epsilon$ is a small constant used to avoid the division by $0$. It is usually taken to be $10^{-8}$. 

In this paper, we mainly focus on the full batch setting, i.e., $\nabla f(\bx_t)$ is the full gradient, for which the dynamical behavior is already rather complex. We also show that our  observations of the full batch setting also apply to the stochastic setting if the batch size is relatively large.  
The systematic investigation of the influence of batch size is left to future work.

\subsection{Continuous-time limits}
RMSProp and Adam can be studied by considering the limiting ordinary differential equations (ODE) obtained by taking the learning rate $\eta$ to $0$. However, 
different limiting ODEs are obtained when  the hyper-parameters are scaled differently. 

If the momentum factors are kept fixed, then as $\eta\rightarrow0$, the memory effect diminishes, because in each discrete iteration we lose the same amount of memory but one iteration 
occupies a shorter and shorter time. In this case, the continuous-time  limit for both RMSprop and Adam are the following dynamics
\begin{equation}\label{eqn: ode1}
    \dot{\bx} = -\frac{\nabla f(\bx)}{|\nabla f(\bx)|+\epsilon}.
\end{equation}
Since $\epsilon$ is a small value, this dynamics is close to the signGD flow:
\begin{equation}\label{eqn: continuous-time-signGD}
    \dot{\bx} = -\textrm{sign}(\nabla f(\bx)).
\end{equation}

\begin{proposition}\label{thm: limit1}
Assume that $\nabla f$ is bounded and Lipschitz continuous, i.e. there exists constants $M$ and $L$ such that $\|\nabla f(\bx_1)\|\leq M$ and $\|\nabla f(\bx_1)-\nabla f(\bx_2)\|\leq L\|\bx_1-\bx_2\|$ hold for any $\bx_1$ and $\bx_2$. Let $\{\bx_k^\eta\}$, $k=0,1,2,\cdots$ be the solution given by algorithm~\eqref{eqn: rmsprop} or~\eqref{eqn: adam} starting from $\bx_0$, $\bfm_0$ and $\bv_0\geq0$, with learning rate $\eta$ and some fixed $\alpha, \beta\in(0,1)$ and $\epsilon>0$. Let $\bX^\eta(\cdot)$ be a piece-wise constant function of $t\in[0,\infty)$ that satisfies
\begin{equation*}
    \bX^\eta(t) = \bx_k^\eta,\quad for  \,t \in[k\eta, (k+1)\eta).
\end{equation*}
In addition, let $\bx(\cdot)$ be the solution of~\eqref{eqn: ode1} initialized from $\bx_0$. Then, for any $T>0$, we have
\begin{equation}
    \lim_{\eta\rightarrow0}\sup_{t\in[0,T]}\|\bX^\eta(t)-\bx(t)\| = 0.
\end{equation}
\end{proposition}

The proof of the proposition is given in the appendix. Figure~\ref{fig: adap_signgd} provides numerical evidence that RMSprop and Adam are close to signGD in a finite time interval when $\eta$ is small while $\alpha$ and $\beta$ are fixed. The closeness between signGD and RMSprop is also shown in Figure~\ref{fig: signgd_rmsprop} for a synthetic objective function.

Note that using the signGD method to train neural networks can date back to \citep{riedmiller1992rprop}, termed as Rprop algorithm. RMSProp was initially proposed as a stochastic version  of Rprop \citep{Tieleman2012}. This type of connection was also investigated for Adam in  \citep{balles2018dissecting}. In contrast, Proposition \ref{thm: limit1} shows another type of connection:  RMSProp/Adam converges to signGD flow in the limit $\eta\to 0$. This connection does not rely on the stochastic approximation and has not been explored before. 

\begin{figure}[!h]
    \centering
    \includegraphics[width=0.4\columnwidth]{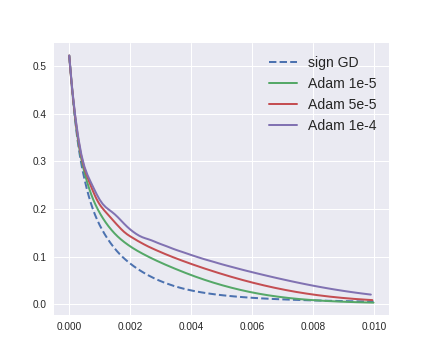}
    \includegraphics[width=0.4\columnwidth]{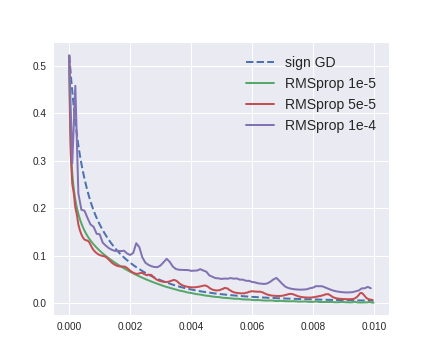}
    \vspace{-4mm}
    \caption{\small The comparison of the training loss curve between signGD  and Adam/RMSprop with different learning rates for the early phase. The x-axis is the time (learning rate$\times$number of iterations) and y-axis denotes the training loss. 
    For Adam, $\beta=0.9$ and $\alpha=0.999$; for RMSprop $\alpha=0.99$. Learning rate of signGD is $10^{-5}$. Experiments conducted on a fully-connected neural network with three hidden layers with widths be $256-128-64$. The training data is taken from $2$ classes of CIFAR10 with $1000$ samples per class.}
    \label{fig: adap_signgd}
\end{figure}

On the other hand, if we want to keep the strength of the memory effect fixed, we have to let $\alpha$ and $\beta$ go to $1$ when $\eta$ tends to $0$. Specifically, let $\alpha=1-a\eta$ and $\beta=1-b\eta$, with $a$ and $b$ being positive constants. Then, it is easy to show that the trajectories of~\eqref{eqn: rmsprop} and~\eqref{eqn: adam} converge to the following ODEs~\eqref{eqn: rmsprop_ode} and~\eqref{eqn: adam_ode}, respectively.

\begin{itemize}
\item {\bf RMSprop flow:}
\begin{equation}
\begin{aligned}
  \dot{\bv} &= a(\nabla f(\bx)^2-\bv) \\
  \dot{\bx} &= -\frac{\nabla f(\bx)}{\sqrt{\bv}+\epsilon} \label{eqn: rmsprop_ode}
\end{aligned}
\end{equation}

\item {\bf Adam flow:}
\vspace{-2mm}
\begin{equation}
\begin{aligned}
  \dot{\bv} &= a(\nabla f(\bx)^2-\bv) \\
  \dot{\bfm} &= b(\nabla f(\bx)-\bfm) \\
  \dot{\bx} &= -\frac{(1-e^{-bt})^{-1}\bfm}{\sqrt{(1-e^{-at})^{-1}\bv}+\epsilon} \label{eqn: adam_ode}
\end{aligned}
\end{equation}
\end{itemize}

The following proposition is a simplification of Theorem 3.2 in~\citep{barakat2018convergence}. Note that this result actually holds for stochastic RMSprop and Adam algorithms, but we will focus on the full-batch setting.

\begin{proposition}\label{thm: limit2}
Under the same condition of $f$ in Proposition~\ref{thm: limit1}, let $\{\bx_k^\eta\}$, $k=0,1,2,\cdots$ be the solution given by algorithm~\eqref{eqn: rmsprop} starting from $\bx_0$ and $\bv_0=0$, with learning rate $\eta$ and $\alpha=1-a\eta$ for a fixed constant $a>0$. Let $\bX^\eta(\cdot)$ be a piece-wise constant vector function of $t\in[0,\infty)$ that satisfies
\begin{equation*}
    \bX^\eta(t) = \bx_k^\eta,\quad \text{for } t \, \in[k\eta, (k+1)\eta).
\end{equation*}
In addition, let $\bx(\cdot)$ be the solution of~\eqref{eqn: rmsprop_ode} initialized from $\bx_0$ and $\bv_0\geq0$. Then, for any $T>0$, we have
\begin{equation}
    \lim_{\eta\rightarrow0}\sup_{t\in[0,T]}\|\bX^\eta(t)-\bx(t)\| = 0.
\end{equation}

Similarly, if $\alpha=1-a\eta$ and $\beta=1-b\eta$ for some constants $a,b>0$, then the same convergence statements hold for
the solutions of~\eqref{eqn: adam} and~\eqref{eqn: adam_ode}.
\end{proposition}


As can be seen from~\eqref{eqn: rmsprop_ode} and~\eqref{eqn: adam_ode}, the smaller the value of $a$ and $b$, the slower the dynamics of $\bv$ (and $\bfm$), and consequently the slower  the whole dynamics. 
Numerical results in Figure~\ref{fig: speed} confirm this. However, it is worth mentioning that this difference in convergence speed does not manifest at the very beginning of the training process. To understand this, consider the dynamics of Adam~\eqref{eqn: adam_ode} with $\bv_0=\bfm_0=0$ and $\epsilon=0$. When $t\ll 1$, we have  
\begin{align}
\bv_t  &\approx (1-e^{-at})\nabla f(\bx_0)^2, \nonumber\\
\bfm_t &\approx (1-e^{-bt})\nabla f(\bx_0). \nonumber
\end{align}
Hence, we have
\begin{equation*}
    \dot{\bx}\approx-\frac{(1-e^{-bt})^{-1}(1-e^{-bt})\nabla f(\bx)}{\sqrt{(1-e^{-at})^{-1}(1-e^{-at})\nabla f(\bx)^2}}  \textrm{sign}(\nabla f(\bx)) \approx \textrm{sign}(\nabla f(\bx_0)),
\end{equation*}
which shows that the initial speed of $\bx$ does not depend on $a$ and $b$.  Rigorously, we have the following proposition, whose proof is deferred to the appendix. 

\begin{proposition}\label{prop: initial}
Given the same conditions of $f$ in Proposition~\ref{thm: limit1}. Let $(\bx(t), \bfm(t), \bv(t))$ be the solution of~\eqref{eqn: adam_ode} starting from $(\bx_0, 0, 0)$ with $\epsilon=0$. Assume $|[\nabla f(\bx_0)]_i|>c$ for some positive constant $c$, where $[\nabla f(\bx_0)]_i$ means the $i$-th element of $\nabla f(\bx_0)$. Then, for any $\tau$ that satisfies $\tau<\frac{c^3}{32M^2L\sqrt{d}}$, we have
\begin{equation}
    \left\|\dot{\bx}(\tau)+\textrm{sign}(\nabla f(\bx(\tau)))\right\| < \frac{54M^3L\sqrt{d}}{c^4}\tau.
\end{equation}
\end{proposition}

\begin{figure*}[!h]
    \centering
    \includegraphics[width=0.45\columnwidth]{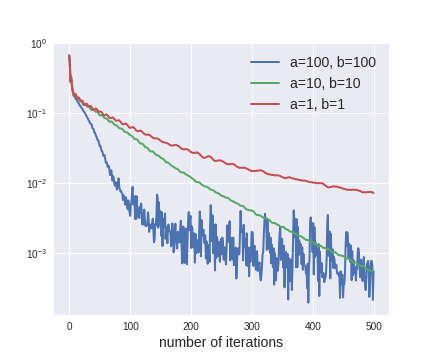}
    \includegraphics[width=0.45\columnwidth]{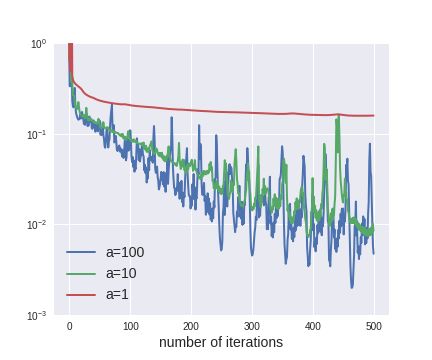}
    \vspace{-4mm}
    \caption{\small How the  values of $a$ and $b$ affect the speed of dynamics.  \textbf{Left:} Adam; \textbf{Right:} RMSprop.  The learning rate is $0.001$ for all the experiments. The model and training data are the same as Figure~\ref{fig: adap_signgd}. 
    One can see that at the early stage of the training (after a very short period from initialization), optimizers with larger $a$ and $b$ converge faster. }
    \label{fig: speed}
\end{figure*}

\section{RMSprop and signGD: Fast convergence and oscillation}

In this section, we focus on RMSprop. Figure~\ref{fig: rmsprop} shows the loss curves and trajectories of RMSprop for a typical
 multi-layer neural network model. There are three obvious features:
\begin{enumerate}
    \item {\bf Fast initial convergence:} the loss curve decreases very fast,  sometimes even super-linearly, at the early stage of the training.
    \item {\bf Small oscillations:} The fast initial convergence is followed by oscillations around the minimum. 
    \item {\bf Large spikes:} Spikes are the sudden increase of the loss values, which are followed by an oscillating recovery. Different from small oscillations, spikes make the loss much larger and the interval between two spikes is also longer. 
\end{enumerate}

We relate the fast initial convergence with the closeness of the RMSprop trajectory to signGD. For the other two features, we attribute them to the instability at the stationary points.

\begin{figure}[!h]
 \centering
 \includegraphics[width=0.4\columnwidth]{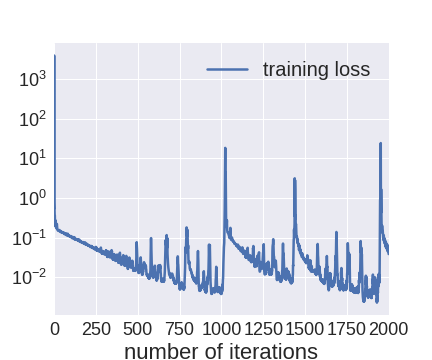}
 \hspace{-5mm}
 \includegraphics[width=0.4\columnwidth]{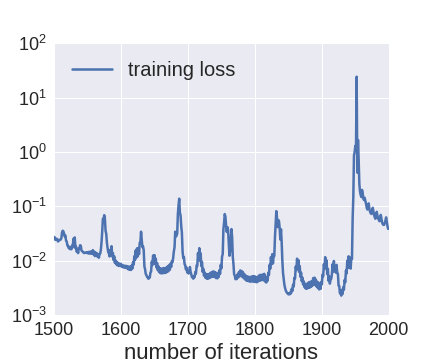}
 \caption{The loss curves of training a neural network model on CIFAR-10 dataset with RMSprop. Model and data the same as Figure~\ref{fig: loss}. The learning rate is 1e-3, and $\alpha=0.99$. $2000$ iterations are run. {\bf Left:} The whole training loss curves {\bf Right:} The training loss of the last $500$ iterations.}
 \label{fig: rmsprop}
\end{figure}

\paragraph{Fast initial convergence} 
As discussed in the last section,  when $\eta$ tends to $0$ while $\alpha$ stays fixed, RMSprop tends to signGD. So the loss curve of RMSprop and signGD align  well during the  initial phase as shown in Figure \ref{fig: adap_signgd}. 
Figure~\ref{fig: signgd_rmsprop} shows the loss curves of both signGD and RMSprop on a quadratic objective function. Their behaviors are similar
--- they both experience fast initial convergence and then the loss stops decreasing. For this reason, we will study the fast initial convergence of RMSprop with the help of signGD. Under the PL condition, the following proposition shows that signGD flow can reach the global minimum in finite time. 

\begin{figure}[!h]
    \centering
    \includegraphics[width=0.35\columnwidth]{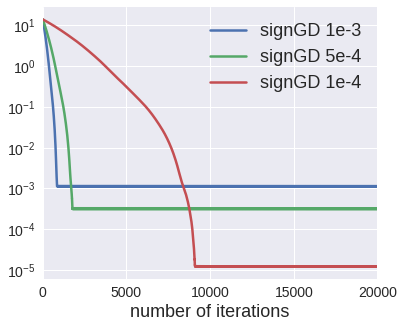}
    \includegraphics[width=0.35\columnwidth]{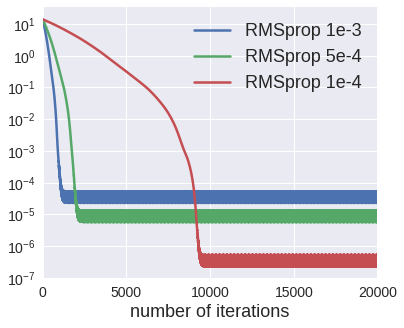}
    \vspace{-5mm}
    \caption{
    \small The loss curves of signGD (\textbf{Left}) and RMSprop (\textbf{Right}) with varying learning rates  for a quadratic objective function $f$. Here, $f(\bx)=\frac{1}{2}\bx^T A\bx$. $A=UU^T$ with $U\in\RR^{10\times 10}$ and $U_{i,j}\stackrel{iid}{\sim} \mathcal{N}(0,1)$. 
   For RMSprop, $\alpha$ is fixed to be $0.9$.
   }
    \label{fig: signgd_rmsprop}
\end{figure}

\begin{proposition}\label{pro: finite-time-stop}
Assume that  the objective function satisfies the Polyak-Lojasiewicz (PL) \citep{polyak1963gradient} condition: $\|\nabla f(\bx)\|_2^2 \geq \mu (f(\bx)-f(\bx^*))$ for any $\bx$. Here $\bx^*=\text{argmin}_{\bx} f(\bx)$.
Let $\bx(t)$ be the solution of the signGD flow \eqref{eqn: continuous-time-signGD}. Then, we have
\begin{equation*}
    f(\bx(t))  - f(\bx^*) \leq \left(\sqrt{f(\bx_0)-f(\bx^*)} - \frac{\sqrt{\mu}}{2}t\right)^2.
\end{equation*}
The signGD flow will stop within $T\leq 2\sqrt{\frac{f(\bx_0)-f(\bx^*)}{\mu}}$.
\end{proposition}

\begin{proof}
We have
\begin{equation*}
\begin{aligned}
  \frac{d}{dt}f(\bx(t)) &= -\langle \textrm{sign}(\nabla f(\bx(t))), \nabla f(\bx(t)) \rangle = -\|\nabla f(\bx(t))\|_1\\ 
  &\leq -\|\nabla f(\bx(t))\|_2 \leq -\sqrt{\mu (f(\bx(t))-f(\bx^*))}.
 \end{aligned}
\end{equation*}
Hence, we have
$
    \frac{d}{dt}\sqrt{f(\bx(t))-f(\bx^*)}\leq -\frac{\sqrt{\mu}}{2},
$
which implies \[
    f(\bx(t)) - f(\bx^*)\leq \left(\sqrt{f(\bx_0)-f(\bx^*)} - \frac{\sqrt{\mu}}{2}t\right)^2.
\]
\end{proof}
 
Consider an one-dimensional objective function $f(x)=\varepsilon x^2$ and the signGD flow starting from $x_0>0$. The signGD flow is given by $\dot{x}(t)=-1$, which will stop at $T_0^*=x_0$. For this example, the PL constant is  $\mu=\|\nabla f(x)\|/f(x)=(2\varepsilon x)^2/(\varepsilon x^2)=4\varepsilon$, which leads to the predicted stopping time $T_0=2\sqrt{x_0^2/(4\varepsilon)}=x_0\varepsilon^{-1/2}$. Consequently, the prediction of Proposition \ref{pro: finite-time-stop} is exact when $\varepsilon=1$ but becomes very loose when $\varepsilon\ll 1$ or $\varepsilon\gg 1$. The intuition is that dynamics of signGD is actually independent of the curvature $\varepsilon$. This is quite different from gradient descent.


\paragraph{Instability at the stationary points}
Intuitively, both the small oscillations and the spikes are caused by the (near) singularity of the dynamics at the stationary points of the objective function. For adaptive gradient methods, at the stationary points, we have $\bv=0$. Hence, if the dynamics is linearized around the stationary point, the Jacobian will have very big eigenvalues (at the order of $O(\frac{1}{\epsilon})$). For example, linearizing continuous RMSprop~\eqref{eqn: rmsprop_ode} around some stationary point $(\bx^*, 0)$ gives
\begin{align*}
    \dot{\bx} &= -\frac{\nabla^2 f(\bx^*)}{\epsilon}(\bx-\bx^*), \\
     \dot{\bv} &= -a\bv.
\end{align*}
The Jacobian of the above linearized dynamics at $(x^*,0)$ is 
\begin{equation}
    \left[\begin{array}{cc}
    -\frac{\nabla^2 f(\bx^*)}{\epsilon} & 0 \\ 0 & -aI \\
    \end{array}\right],
\end{equation}
whose eigenvalues are $-\lambda/\epsilon$ and $-a$, where $\lambda$ is any eigenvalue of the Hessian $\nabla^2 f(\bx^*)$. Therefore, this stationary point is nearly singular. This is not a problem for the continuous dynamics. However, for the discrete dynamics, the stationary point is unstable unless the learning rate is smaller than  $\frac{2\epsilon}{ \lambda_{\max}}\ll 1$. This implies that the discrete dynamics with a learning rate larger than $O(\epsilon)$ cannot converge to that stationary point, and instead, it may converge to some periodic trajectories or just oscillate around the stationary point. This analysis also holds for Adam.


\begin{remark}
A rough way to understand the near singularity at the stationary points is to view the iterations as a GD with very large learning rate. Specifically, when $\bx_t$ is close to a stationary point, $\bv$ is very small, and we have 
\begin{equation}\label{eqn: approx_gd}
    \bx_{t+1}=\bx_t-\eta\frac{\nabla f(\bx_t)}{\sqrt{\bv_{t+1}}+\epsilon} \approx \bx_t-\frac{\eta}{\epsilon}\nabla f(\bx_t).
\end{equation}
Hence, $\bx_t$ does not converge to the stationary point unless $\lambda_{\max}(\nabla^2 f(\bx_t))\leq O(\varepsilon)$, i.e.,  the landscape is extremely flat.
\end{remark}

According the above analysis, oscillations may not happen if the landscape around the stationary point is sufficiently flat. This happens to be the case of the cross entropy loss function, for which   $\lambda_{\max}(\nabla^2 f(\bx_t))\to 0$ as $\bx_t$ approaches the minimum.

For low dimensional strongly convex objective functions, RMSprop can converge to a $2$-periodic solution. For example, if the objective function is $f(x)=\frac{1}{2}x^2$, then the $2$-periodic solution is an oscillation between $x=\frac{\eta}{2}$ and $x=-\frac{\eta}{2}$, where $\eta$ is the learning rate. Figure~\ref{fig: 2periodic} shows the convergence to this $2$-periodic solution. This gives us a toy example of the small oscillations around the minimum.

\begin{figure}[!h]
    \centering
    \includegraphics[width=0.5\columnwidth]{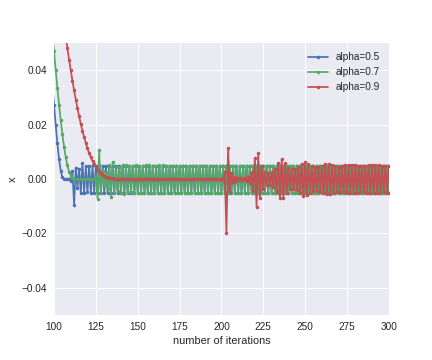}
    \vspace{-5mm}
    \caption{\small The trajectory of RMSprop for the $1$-dimensional quadratic function $f(x)=\frac{x^2}{2}$ for different values of $\alpha$. $\eta=0.01$. One sees that  all the trajectories eventually converge to the $2$-periodic solution at $x=\frac{\eta}{2}$ and $x=-\frac{\eta}{2}$.}
    \label{fig: 2periodic}
\end{figure}

For more complicated objective functions, such as high-dimensional quadratic function, or the loss function of neural network models, the RMSprop trajectories show more complicated oscillation patterns, such as the spikes. We will take a closer look at the large spikes in the next section, where we see that Adam is also vulnerable to spikes.

\section{Adam: performances for different values of a and b}
The dynamic behavior of Adam is more complicated than RMSprop since it is influenced by $2$ hyper-parameters. Different combinations of $\alpha$ and $\beta$ (or $a$ and $b$) can lead to different dynamic patterns. 
To rule out the influence of the learning rate,  we will consider $a$ and $b$ instead of $\alpha$ and $\beta$. As is mentioned 
before, $\alpha$ and $\beta$ are given by $a$ and $b$ through $\alpha=1-a\eta$ and $\beta=1-b\eta$. As we have seen in Proposition \ref{thm: limit1} , when $a$ and $b$ are sufficiently large compared to $\eta$, Adam behaves like signGD. For relatively small $a$ and $b$, through extensive numerical experiments, we have found  that  there are roughly three different regimes of qualitative patterns in the parameter space (see Figure~\ref{fig: adam_patterns}):
\begin{figure*}[!h]
    \centering
    \includegraphics[width=0.32\textwidth]{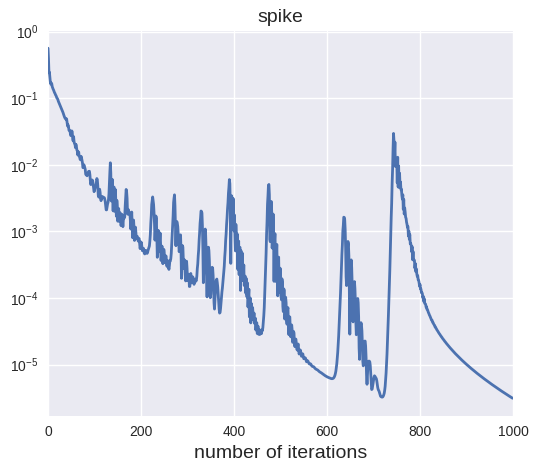}
    \includegraphics[width=0.32\textwidth]{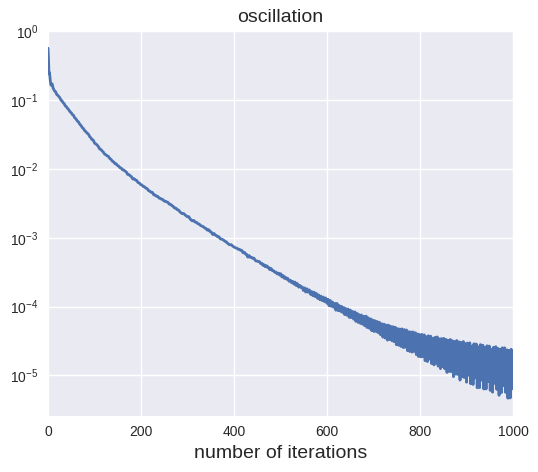}
    \includegraphics[width=0.32\textwidth]{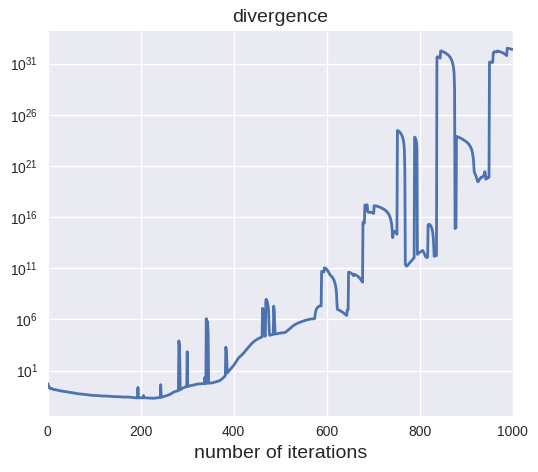}
    \includegraphics[width=0.32\textwidth]{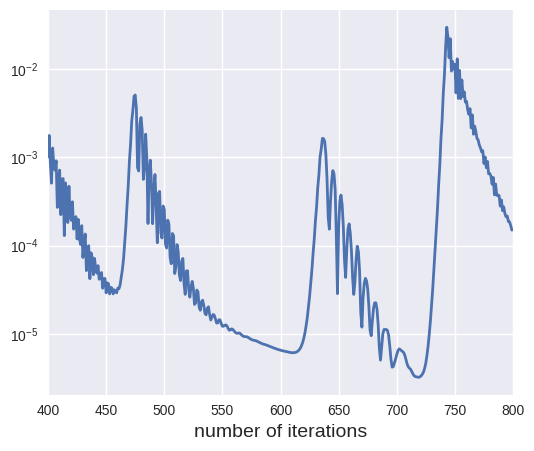}
    \includegraphics[width=0.32\textwidth]{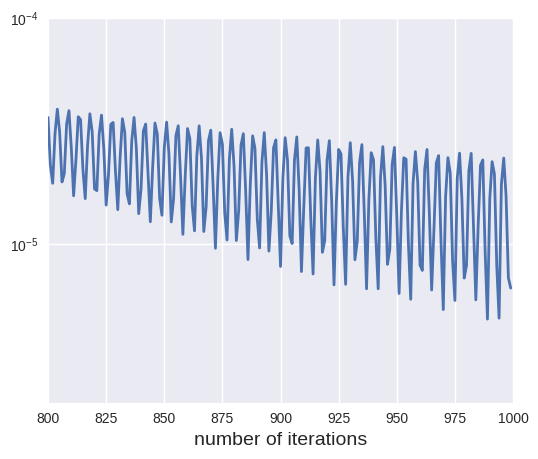}
    \includegraphics[width=0.32\textwidth]{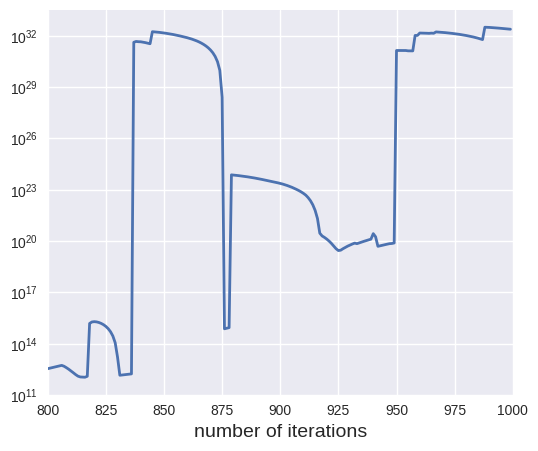}
    \includegraphics[width=0.32\textwidth]{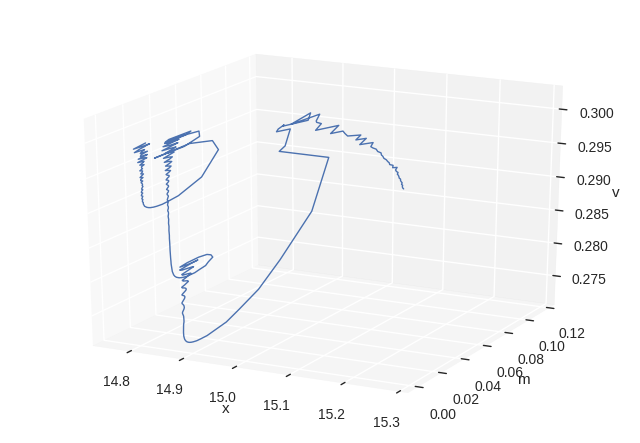}
    \includegraphics[width=0.32\textwidth]{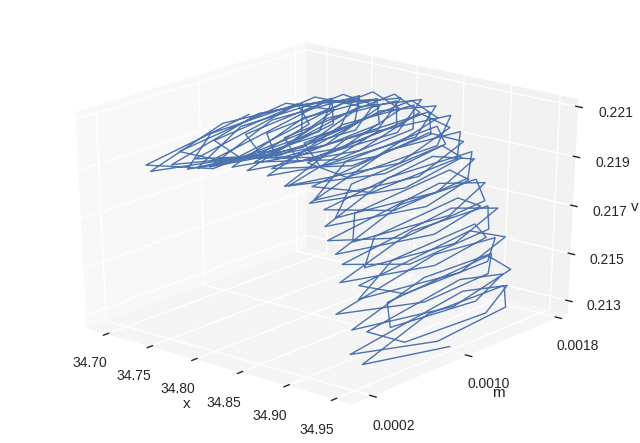}
    \includegraphics[width=0.32\textwidth]{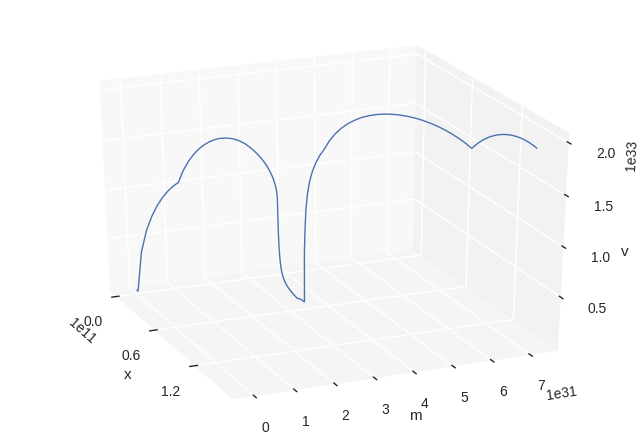}
    \caption{The three typical behavior patterns for Adam and the trajectories in the state space of $(\|\bx\|, \|\bfm\|, \|\sqrt{\bv}\|)$. $\eta=0.001$. The model and the training data are the same as Figure~\ref{fig: adap_signgd}. The first row shows the loss curve of totally $1000$ iterations, the second row shows part of the loss curve (the last $200$ iterations for oscillation and divergence regimes, and $400-800$ iterations for the spike regime), the bottom row shows the state space trajectory in the same period shown in the second row. {\bf Left:} $a=1$, $b=100$, large spikes appear in the loss curve; {\bf Middle:} $a=10$, $b=10$, the loss is small and oscillates very fast, and the amplitude of the oscillation is also small; {\bf Right:} $a=100$, $b=1$, the loss is large and blows up.}
    \label{fig: adam_patterns}
\end{figure*}
\begin{enumerate}
    \item {\bf The spike regime} happens when $b$ is sufficiently larger than $a$. In this regime, large spikes appear in the loss curve, which makes the optimization process unstable. By observations, spikes do not prevent the algorithm from achieving a small training loss, but they make the loss curve unstable by frequently driving the training loss to large values.
    
    \item {\bf The oscillation regime} happens when $a$ and $b$ have similar magnitude (or in the same order). In this regime, the loss curve exhibits fast and small oscillations. A small and stable loss curve can be achieved in this regime.
    
    \item {\bf The divergence regime} happens when $a$ is sufficiently larger than $b$. In this regime, the loss curve is unstable and usually diverges after a period of training. This regime should be avoided in practice since the training loss stays large. 
\end{enumerate}
In Figure~\ref{fig: adam_patterns}, we show one typical loss curve for each regime for a typical neural network model. We also show typical trajectories in the state space of $(\|\bx\|, \|\bfm\|, \|\sqrt{\bv}\|)$ for the three regimes. These trajectories are also qualitatively different for different regimes.

Next, we study the transition between  the different regimes and the training loss behavior in different regimes. 
To this end, we carried out experiments 
for a multi-layer neural network model on the Fashion-MNIST dataset, with  different values of $a$ and $b$ until the behavior of the training loss curve stabilizes. The left panel of Figure~\ref{fig: heatmap_nn} shows the heatmap of the average loss value of the last $1000$ iterations. The right panel of Figure~\ref{fig: heatmap_nn} shows the classification of the behavior of the training curve into three different categories (oscillations, spikes, and divergence).  

From these figures, we see that in the divergence regime, the training loss does not perform well (actually in some cases it may even blow up).
Hence this regime should be avoided in practice. In the oscillation regime, the loss values are small and quite robust to the change of hyper-parameters. Therefore this is the regime that should be preferred in practice. This is the regime when $a\approx b$.
\begin{figure}[!h]
\centering
\includegraphics[width=0.42\textwidth]{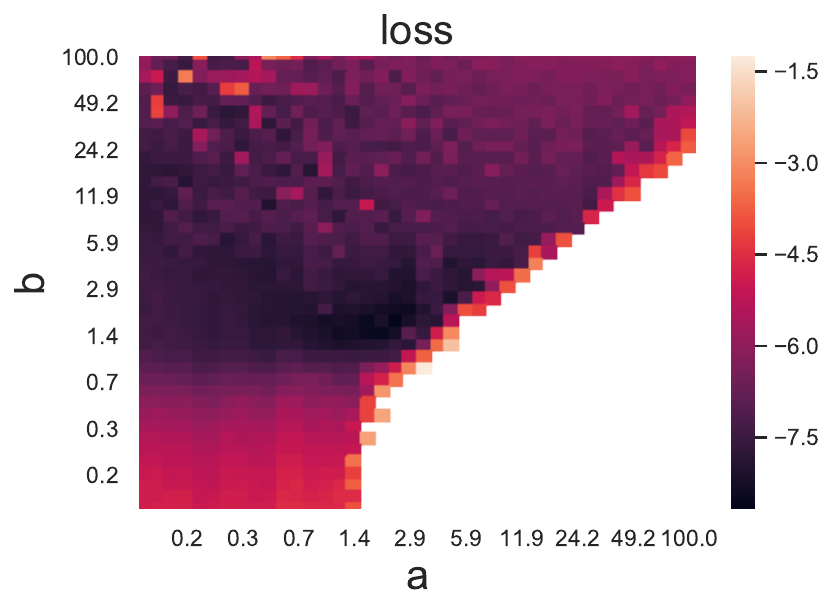}
\includegraphics[width=0.4\textwidth]{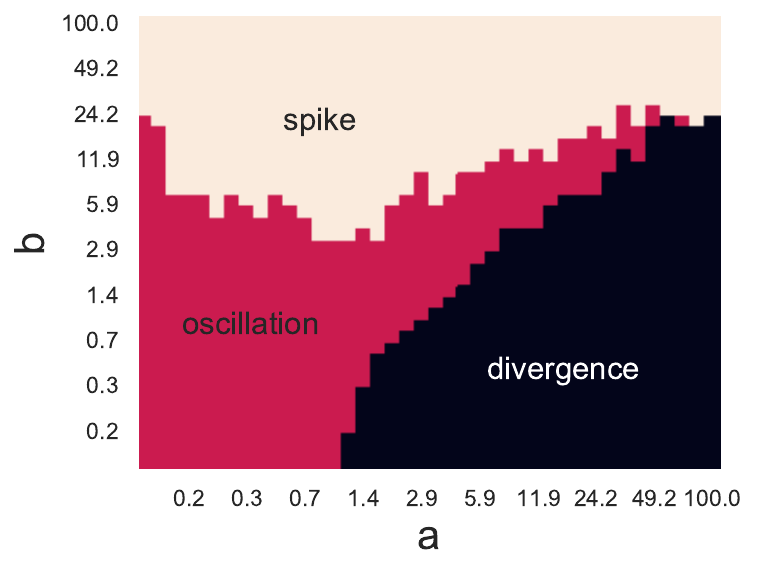}
\vspace{-4mm}
\caption{ \small Train a neural network to fit the FashionMNIST dataset with Adam optimizers with varying momentum factors. The model is a fully connected network of $6$ hidden layers and the width of each layer is $500$. The learning rate of is 1e-3.  {\bf Left:} Heatmap of average training loss  over the last $1000$ iterations. The loss is shown in the logarithmic scale. $a$ and $b$ range from $0.1$ to $100$ and are also shown in the logarithmic scale. {\bf Right:} The classification of the different dynamical behaviors of the loss curve. } 
\label{fig: heatmap_nn}
\end{figure}

\vspace{-2mm}
\subsection{Training ResNets on CIFAR10}
The above investigation suggests that Adam performs better when $\alpha\approx\beta$. Here we provide  further support  by considering a more realistic problem:   training a ResNet18 \citep{he2016deep} on CIFAR10 using stochastic Adam with a large batch size. The results are shown in Figure \ref{fig: cifar}. We see that with the default parameters ($\beta=0.9, \alpha=0.999$), there are are  large spikes during the late phase of training. In contrast, when $a\approx b$, Adam converges very smoothly and is also faster than using the default parameters. 
\begin{figure}[!h]
\centering
\includegraphics[width=0.42\columnwidth]{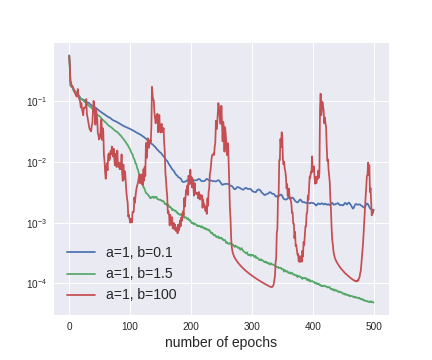}
\vspace{-4mm}
\caption{ \small The training loss curve of stochastic Adam on a ResNet18 model and CIFAR-10 dataset. The learning rate is 1e-3. The red line shows the
results of using the default hyper-parameters setting ($\beta=0.9, \alpha=0.999$). $1000$ samples are taken from each class to form the training dataset. The network is a standard ResNet18 used in \citep{he2016deep} but with number of channels halved. The batch size is $1000$. }
\label{fig: cifar}
\end{figure}

\subsection{Solving Poisson equation}
Consider the Poisson equation with Dirichlet boundary condition
\begin{align}
\notag -\Delta u &= f \, \text{ in } \Omega,\quad
u= g \, \text{ on } \,\partial \Omega.
\end{align}
Let $u(\cdot;\theta)$ denote the parameterized model. 
Deep Galerkin Method (DGM) \citep{sirignano2018dgm} looks for the solution that minimizes the following objective function
\begin{align}
    \hat{I}(\theta)=\frac{1}{n_d}\sum_{i=1}^{n_d}(\Delta u(\bx_i;\theta)+f(\bx_i))^2 + \frac{1}{n_b}\sum_{j=1}^{n_b} (u(\tilde{\bx}_j;\theta)-g(\tilde{\bx}_j))^2,
\end{align}
where $\{\bx_i\}_{i=1}^{n_d}$ and $\{\tilde{\bx}_j\}_{j=1}^{n_b}$ are samples uniformly drawn from $\Omega$ and $\partial \Omega$, respectively.

Here, we consider two examples: 
\begin{itemize}
    \item $\Omega=(0,1)^4, f(\bx)=0, g(\bx)=x_1x_2+x_3x_4$. In this case, the solution is $u^*(\bx)=x_1x_2+x_3x_4$. $u(\cdot;\theta)$ is parameterized using a 3-layer fully-connected networks with the architecture being 4-200-200-1 and the Tanh activation function is applied.
    \item $\Omega=\{(x_1,x_2): x_1^2+x_2^2<1\}, f(x_1,x_2)=1, g(x_1,x_2)=0$. The solution in this case is $u^*(\bx)=\frac{1}{4}(x_1^2+x_2^2-1)$. $u(\cdot;\theta)$ is parameterized using a 5-layer fully connected network, whose architecture is 2-10-10-10-10-1. The GELU activation function \citep{hendrycks2016gaussian} is applied.
\end{itemize}
For each example, we uniformly sample $2000$ points in $\Omega$ and extra $2000$ points on $\partial \Omega$ to form the training set. 
 Figure \ref{fig: pde} shows the training curves of Adams with various $a$'s and $b$'s. The learning rate is fixed to be 5e-4. One can see that with the default hyperparameters, Adam inevitably  endures large spikes during the late phase of training. In contrast,  when $a\approx b$, the training curve becomes much more stable and faster, although there still exist very small oscillations during the very late phase of training. It is also expected that Adam performs the worst when $a>b$.

\begin{figure}[!h]
    \centering
    \subfigure[$\Omega=(0,1)^4, f=0, g(\bx)=x_1x_2+x_3x_4$.]{
    \includegraphics[width=0.4\textwidth]{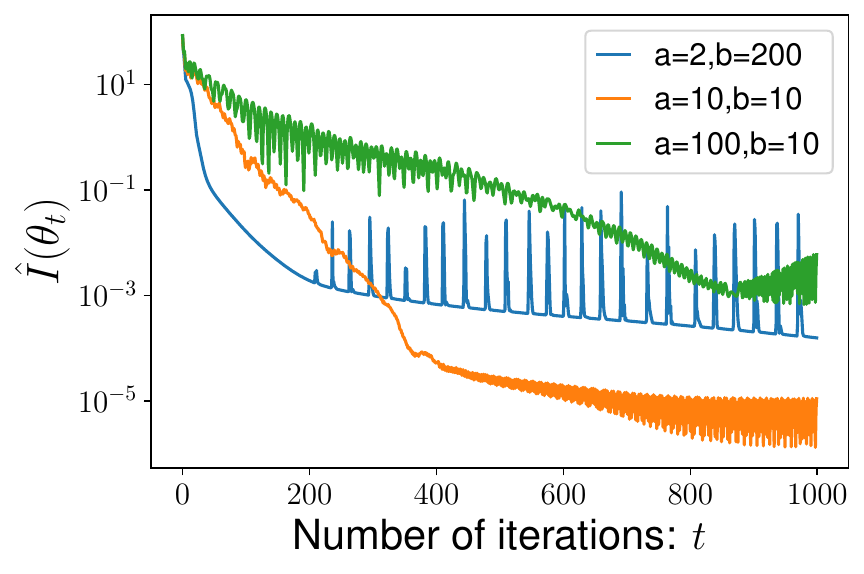}
    \includegraphics[width=0.4\textwidth]{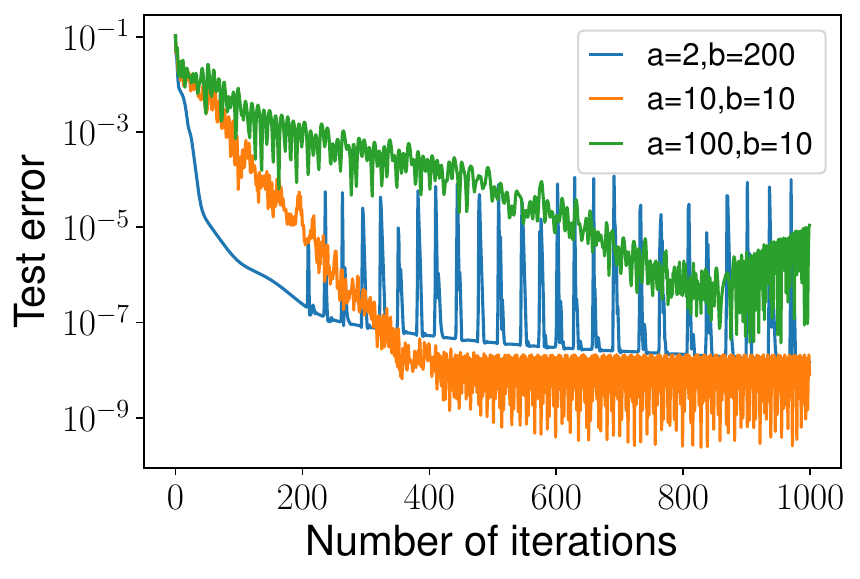}
    }
    \subfigure[$\Omega=\{(x_1,x_2):x_1+x_2<1\}, f=1, g=0$]{
    \includegraphics[width=0.4\textwidth]{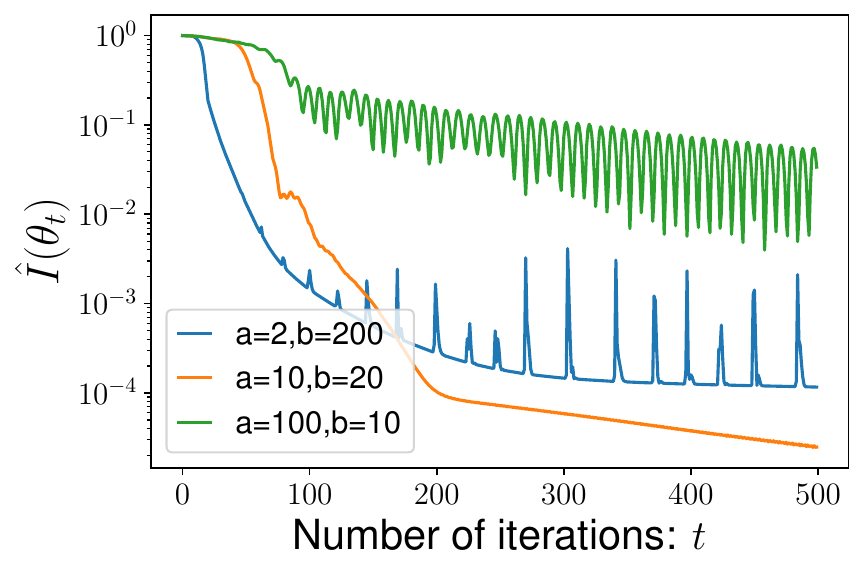}
    \includegraphics[width=0.4\textwidth]{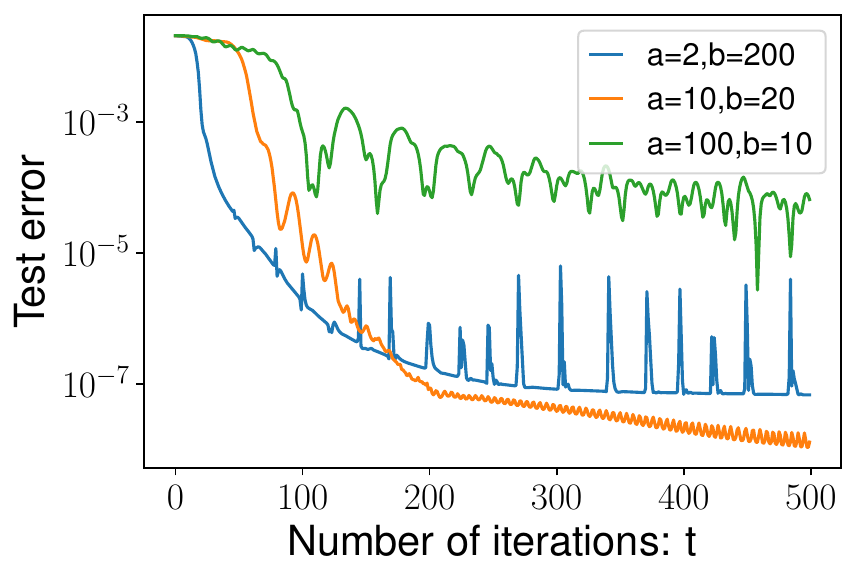}
    }
    \vspace{-4mm}
    \caption{ \small Solving two Poisson equations with the deep Galerkin method (DGM). Adam optimizers with various $a$'s and $b$'s  are used to optimize the objective functions. The learning rate is fixed to be 5e-4. Note that the case of $a=2,b=200$ corresponds to the default setting, i.e. $\beta=0.9,\alpha=0.999$. \textbf{Left: } The dynamics of the objective function of DGM. \textbf{Right: } The dynamics of the test error. Here the test error is $\EE_{\bx}[(u(\bx;\theta)-u^*(\bx))]$, which is estimated by the emprical mean over $10000$ extra samples, independently drawn from $\text{Unif}(\Omega)$.}
    \label{fig: pde}
\end{figure}

\section{Discussion}
In this paper, we reported the results of some systematical investigation on the  dynamic behavior of  adaptive gradient algorithms, 
particularly RMSprop and Adam. Three typical phenomena---fast initial convergence, small oscillation, and large spikes---are observed and analyzed.
The influence of the choice of the hyper-parameters on the dominant training  behavior is also investigated. 

It is worth noting that the investigation in this paper focuses on the full-batch setting. However, the result in Figure \ref{fig: cifar} provides some evidence  to show that the phenomena revealed here should also be of relevance for the stochastic setting when the batch size is relatively large. The systematic study of the influence of batch size, especially in the small-batch regime, is left to future work.

There are still many other important open questions.
For example, learning rate decay is a common practice used in training large neural networks. When 
performing learning rate decay, usually one does not change the values of $\alpha$ and $\beta$.
This makes the effective $a$ and $b$  larger, pushing the optimizer to the signGD-like regime. Another choice is to adaptively tune $\alpha, \beta$ such that $a$ and $b$ are kept fixed. It is interesting to see the comparison of the two strategies.   

This paper focuses on optimization.
For machine learning problems, another important consideration when implementing optimization algorithms is the generalization performance. It has been reported that the solutions found by adaptive gradient algorithms usually perform a bit worse than those found by SGD 
in terms of generalization (see~\citep{wilson2017marginal}). 
The study of generalization performance of adaptive gradient algorithms is left to future work. 

{\small 
\bibliography{main_msml.bbl}
}

\newpage
\appendix
\newcommand{\xketa}{\bx_k^\eta}
\newcommand{\xkpeta}{\bx_{k+1}^\eta}
\newcommand{\txketa}{\tilde{\bx}_k^\eta}
\newcommand{\txkpeta}{\tilde{\bx}_{k+1}^\eta}

\section{Proof of Proposition~\ref{thm: limit1}}\label{app: proof1}
Here we prove Proposition~\ref{thm: limit1}. For that purpose, we show that for any $T>0$ and $\tau>0$, there exists an $\eta_{T,\tau}$, such that as long as $\eta<\eta_{T,\tau}$ we have
\begin{equation}
    \sup_{t\in[0,T]}\|\bX^\eta(t)-\bx(t)\| < \tau.
\end{equation}
In the following we focus on RMSprop. The proof for Adam is similar.

First, let $K$ be a positive integer whose value will be specified later, and let $\txketa=\bx(k\eta)$. Then, for $\bx_K^\eta$ and $\tilde{\bx}_K^\eta$ we have
\begin{equation*}
    \|\bx_K^\eta-\bx_0\|\leq \eta\sum\limits_{i=0}^{K-1}\left\|\frac{\nabla f(\bx_i)}{\sqrt{\bv_{i+1}}+\epsilon}\right\|\leq \eta\sum\limits_{i=0}^{K-1}\frac{M}{\epsilon}=\frac{\eta KM}{\epsilon},
\end{equation*}
and
\begin{equation*}
    \|\tilde{\bx}_K^\eta-\bx_0\|\leq\int_0^{K\eta}\left\|\frac{\nabla f(\bx(t))}{|\nabla f(\bx(t))|+\epsilon}\right\|dt\leq \frac{\eta KM}{\epsilon}
\end{equation*}
Therefore, 
\begin{equation}
    \|\bx_K^\eta-\tilde{\bx}_K^\eta\|\leq \frac{2\eta KM}{\epsilon}.
\end{equation}

Next, for $k\geq K$, we have
\begin{equation*}
    \xkpeta-\txkpeta = (\xketa-\txketa) + \left(\int_{k\eta}^{(k+1)\eta}\frac{\nabla f(\bx(t))}{|\nabla f(\bx(t))|+\epsilon}dt - \eta\frac{\nabla f(\xketa)}{\sqrt{\bv_{k+1}}+\epsilon}\right).
\end{equation*}
Let 
\begin{equation*}
    \Delta = \int_{k\eta}^{(k+1)\eta}\frac{\nabla f(\bx(t))}{|\nabla f(\bx(t))|+\epsilon}dt - \eta\frac{\nabla f(\xketa)}{\sqrt{\bv_{k+1}}+\epsilon},
\end{equation*}
then 
\begin{equation}
    \|\xkpeta-\txkpeta\| \leq \|\xketa-\txketa\| + \|\Delta\|. 
\end{equation}
Next we estimate $\|\Delta\|$. First we have
\begin{align*}
\Delta & = \int_{k\eta}^{(k+1)\eta} \frac{\nabla f(\bx(t))}{|\nabla f(\bx(t))|+\epsilon} - \frac{\nabla f(\xketa)}{\sqrt{\bv_{k+1}}+\epsilon} dt \\
  & = \int_{k\eta}^{(k+1)\eta} \nabla f(\bx(t))\left(\frac{1}{|\nabla f(\bx(t))|+\epsilon} - \frac{1}{\sqrt{\bv_{k+1}}+\epsilon}\right)dt + \int_{k\eta}^{(k+1)\eta} \frac{\nabla f(\bx(t))-\nabla f(\xketa)}{\sqrt{\bv_{k+1}}+\epsilon} dt \\
  & := I + J.
\end{align*}
For $J$, we have
\begin{align}
\|J\| &\leq \frac{1}{\epsilon}\int_{k\eta}^{(k+1)\eta} \|\nabla f(\bx(t))-\nabla f(\xketa)\| dt \nonumber\\
  & \leq \frac{L}{\epsilon}\int_{k\eta}^{(k+1)\eta} \|\bx(t)-\xketa\|dt \nonumber\\
  & \leq \frac{L}{\epsilon}\int_{k\eta}^{(k+1)\eta} (\|\bx(t)-\txketa\|+\|\txketa-\xketa\|)dt \nonumber\\
  & \leq \frac{L}{\epsilon} \left(\frac{\eta^2M}{\epsilon} + \eta\|\txketa-\xketa\| \right) \nonumber\\
  & = \frac{\eta L}{\epsilon}\|\txketa-\xketa\| + \frac{\eta^2LM}{\epsilon^2}. \label{eqn: est_J}
\end{align}
For $I$, we have
\begin{align}
\|I\| & \leq \int_{k\eta}^{(k+1)\eta} \left\| \frac{|\nabla f(\bx(t))|}{|\nabla f(\bx(t))|+\epsilon}\frac{\sqrt{\bv_{k+1}}-|\nabla f(\bx(t))|}{\sqrt{\bv_{k+1}}+\epsilon} \right\|dt \nonumber \\
  & \leq \int_{k\eta}^{(k+1)\eta} \left\| \frac{\sqrt{\bv_{k+1}}-|\nabla f(\bx(t))|}{\sqrt{\bv_{k+1}}+\epsilon} \right\|dt \nonumber\\ 
  & \leq \frac{1}{\epsilon} \int_{k\eta}^{(k+1)\eta} \|\sqrt{\bv_{k+1}}-|\nabla f(\xketa)|\|dt + \frac{1}{\epsilon} \int_{k\eta}^{(k+1)\eta} \||\nabla f(\xketa)|-|\nabla f(\bx(t))|\|dt \label{eqn: est_I}
\end{align}
The second term in~\eqref{eqn: est_I} can be estimated in a similar way as $\|J\|$, and  it can be bounded by
\begin{equation*}
    \frac{\eta L}{\epsilon}\|\txketa-\xketa\| + \frac{\eta^2LM}{\epsilon^2}.
\end{equation*}
For the first term of~\eqref{eqn: est_I}, use the fact that $(a-b)^2\leq a^2-b^2$ for any $a\geq b\geq0$, we have
\begin{align}
& \frac{1}{\epsilon} \int_{k\eta}^{(k+1)\eta} \|\sqrt{\bv_{k+1}}-|\nabla f(\xketa)|\|dt \nonumber\\
&\leq \frac{\eta}{\epsilon}\|\bv_{k+1}-\nabla f^2(\xketa)\|^{\frac{1}{2}} \nonumber\\
  & = \frac{\eta}{\epsilon}\left\| (1-\alpha)\nabla f^2(\bx_k^\eta) + \alpha(1-\alpha)\nabla f^2(\bx_{k-1}^\eta) + \cdots + \alpha^k(1-\alpha)\nabla f^2(\bx_0^\eta) - \nabla f^2(\bx_k^\eta) \right\|^{\frac{1}{2}} \nonumber\\
  &\leq \frac{\eta}{\epsilon}\left(\left\|(1-\alpha)\sum\limits_{i=0}^{K-1}\alpha^i(\nabla f^2(\bx_{k-i}^\eta)-\nabla f^2(\bx_k^\eta))\right\| + 2\alpha^K M\right)^{\frac{1}{2}} \nonumber\\
  &\leq \frac{2\eta^{3/2}MK^{1/2}}{\epsilon^{3/2}}+\frac{2\eta M\alpha^{K/2}}{\epsilon}.
\end{align}
Hence  we have
\begin{equation}
    \|I\|\leq \frac{\eta L}{\epsilon}\|\txketa-\xketa\| + \frac{\eta^2LM}{\epsilon^2} + \frac{2\eta^{3/2}MK^{1/2}}{\epsilon^{3/2}}+\frac{2\eta M\alpha^{K/2}}{\epsilon}. \label{eqn: est_I2}
\end{equation}
Combining~\eqref{eqn: est_I2} with~\eqref{eqn: est_J} we get the estimate of $\Delta$:
\begin{equation}
    \|\Delta\|\leq \frac{2\eta L}{\epsilon}\|\txketa-\xketa\| + \frac{2\eta^2LM}{\epsilon^2} + \frac{2\eta^{3/2}MK^{1/2}}{\epsilon^{3/2}}+\frac{2\eta M\alpha^{K/2}}{\epsilon}.
\end{equation}
Hence
\begin{equation}
    \|\xkpeta-\txkpeta\| \leq (1+\frac{2\eta L}{\epsilon})\|\xketa-\txketa\| +  \frac{2\eta^2LM}{\epsilon^2} + \frac{2\eta^{3/2}MK^{1/2}}{\epsilon^{3/2}}+\frac{2\eta M\alpha^{K/2}}{\epsilon}.
\end{equation}

Finally, by Gronwall's inequality,  we have
\begin{align}
\|\xketa-\txketa\| & \leq \left(1+\frac{2\eta L}{\epsilon}\right)^{k-K}\|\bx_K^\eta-\tilde{\bx}_K^\eta\| +  \left(1+\frac{2\eta L}{\epsilon}\right)^{k-K}\left(\frac{\eta M}{\epsilon} + \frac{\eta^{1/2}MK^{1/2}}{\epsilon^{1/2}L}+\frac{ M\alpha^{K/2}}{L}\right) \nonumber\\
  & \leq \left(1+\frac{2\eta L}{\epsilon}\right)^{k}\left(\frac{2\eta KM}{\epsilon}+\frac{\eta^{1/2}MK^{1/2}}{\epsilon^{1/2}L}+\frac{ M\alpha^{K/2}}{L}\right). \label{eqn: gronwall}
\end{align}
We want~\eqref{eqn: gronwall} to hold for $t\leq T$, which means for all $k\leq \frac{T}{\eta}$. For these values of  $k$, we have
\begin{equation}
\|\xketa-\txketa\|\leq e^{\frac{LT}{\epsilon}}\left(\frac{2\eta KM}{\epsilon}+\frac{\eta^{1/2}MK^{1/2}}{\epsilon^{1/2}L}+\frac{ M\alpha^{K/2}}{L}\right).
\end{equation}
Therefore, for any fixed small value $\tau>0$, by taking sufficiently large $K$ and sufficiently small $\eta$, we can achieve
\begin{equation}
    \|\xketa-\txketa\|\leq \frac{\tau}{2},
\end{equation}
for any $0\leq k\leq \lfloor \frac{T}{\eta} \rfloor+1$. Then, if we further let 
\begin{equation*}
    \eta < \frac{\tau\epsilon}{4M},
\end{equation*}
for any $t\in[0,T]$, let $k$ satisfy  $t\in[k\eta, (k+1)\eta)$, we have
\begin{align}
\|\bX^\eta(t)-\bx(t)\| &\leq \|\xketa-\txketa\| + \|\bX^\eta(t)-\txketa\| + \|\bx(t)-\xketa\| \nonumber\\
  &\leq \frac{\tau}{2} + \frac{2\eta M}{\epsilon} \nonumber\\
  &\leq \tau.
\end{align}
This completes the proof.

\section{Proof of Proposition~\ref{prop: initial}}
By~\citep{barakat2018convergence} as well as Proposition~\ref{thm: limit2}, the solution of~\eqref{eqn: adam_ode} exists, and this solution is the limit trajectory of discrete Adam algorithm~\eqref{eqn: adam} with $\eta\rightarrow0$ and $\alpha=1-a\eta$, $\beta=1-b\eta$. Hence, initially we have
\begin{equation*}
    \frac{d\bx(0)}{dt} = -\textrm{sign}(\nabla f(\bx_0)). 
\end{equation*}
By~\eqref{eqn: adam_ode}, we can solve $\bv$ and $\bfm$ involving $\bx$:
\begin{align}
    \bv(t) &= a\int_0^t e^{a(s-t)}(\nabla f(\bx(t)))^2 ds, \nonumber\\
    \bfm(t) &= b\int_0^t e^{b(s-t)}\nabla f(\bx(t)) ds,
\end{align}
and hence for the equation of $\bx$ we have
\begin{equation}\label{eqn:x_dyn}
    \dot{\bx}(t) = -\sqrt{\frac{\int_0^t e^{-as}ds}{\int_0^t e^{a(s-t)}(\nabla f(\bx(t)))^2 ds}}\left(\frac{\int_0^t e^{b(s-t)}\nabla f(\bx(t)) ds}{\int_0^t e^{-bs}ds}\right).
\end{equation}
Let $t^*$ be the first time when some element of $\nabla f(\bx(t))$ becomes smaller than $c/2$, i.e.
\begin{equation*}
    t^* = \inf_t\left\{[\nabla f(\bx(t))]_i\leq \frac{c}{2} \textrm{ for some }i\right\},
\end{equation*}
then for any $t\in[0, t^*]$ we have $[\nabla f(\bx(t))]_i\geq c/2$ for any $i=1,2,...,d$. This together with~\eqref{eqn:x_dyn} implies
\begin{equation}\label{eqn:x_speed}
    \left| [\dot{\bx}]_i \right|\leq \frac{2M}{c}
\end{equation}
for any $i=1,2,...,d$. Then, assume $t^*<\frac{c^2}{4ML}$, we have 
\begin{align*}
|[\nabla f(\bx(t^*))]_i| &\geq |[\nabla f(\bx(0))]_i|-L\|\bx(t^*)-\bx(0)\| \\
    &> |[\nabla f(\bx(0))]_i| - L\frac{c^2}{4ML}\frac{2M}{c} \\
    &\geq c - \frac{c}{2} \\
    &=\frac{c}{2},
\end{align*}
which is contradictory to the definition of $t^*$. Therefore, $t^*\geq\frac{c^2}{4ML}$. Considering $c<M$, we obtain $\tau<t^*$. 

Next, let 
\begin{align}
\br_1(\tau) &= \frac{\int_0^\tau e^{a(s-\tau)}\left((\nabla f(\bx(s)))^2-(\nabla f(\bx(\tau)))^2\right)ds}{\int_0^\tau e^{-as}ds}, \nonumber\\
\br_2(\tau) &= \frac{\int_0^\tau e^{b(s-\tau)}\left(\nabla f(\bx(s))-\nabla f(\bx(\tau))\right)ds}{\int_0^\tau e^{-bs}ds}. \nonumber
\end{align}
Then, we have
\begin{equation*}
   \frac{\int_0^t e^{a(s-t)}(\nabla f(\bx(t)))^2 ds}{\int_0^t e^{-as}ds} = (\nabla f(\bx(\tau)))^2 + \br_1(\tau),
\end{equation*}
and 
\begin{equation*}
    \frac{\int_0^t e^{b(s-t)}\nabla f(\bx(t)) ds}{\int_0^t e^{-bs}ds}= \nabla f(\bx(\tau)) + \br_2(\tau). 
\end{equation*}
Hence, combining~\eqref{eqn:x_dyn}, we have
\begin{equation*}
    \dot{\bx}(\tau) = -\sqrt{\frac{1}{(\nabla f(\bx(\tau)))^2 + \br_1(\tau)}}(\nabla f(\bx(\tau)) + \br_2(\tau)),
\end{equation*}
and then at $\tau$
\begin{align}
\|\dot{\bx}+\textrm{sign}(\nabla f(\bx))\| & = \left\| \frac{\nabla f(\bx)}{\sqrt{(\nabla f(\bx))^2}}-\sqrt{\frac{1}{(\nabla f(\bx))^2 + \br_1}}(\nabla f(\bx) + \br_2) \right\| \nonumber\\ 
 & \leq \left\| \nabla f(\bx)\left(\sqrt{\frac{1}{(\nabla f(\bx))^2}}-\sqrt{\frac{1}{(\nabla f(\bx))^2+\br_1}}\right) \right\| \nonumber\\
 & \ \ + \left\| \sqrt{\frac{1}{(\nabla f(\bx))^2+\br_1}}(\nabla f(\bx)-(\nabla f(\bx)+\br_2)) \right\|
\end{align}

To estimate the above terms, we first estimate $\br_1$ and $\br_2$. For $\br_2$, by the Lipschitz property of the gradient and~\eqref{eqn:x_speed}, we have
\begin{equation*}
\|\nabla f(\bx(s)) - \nabla f(\bx(\tau))\|\leq L\|\bx(s)-\bx(\tau)\| \leq \frac{2ML\sqrt{d}}{c}|\tau-s| \leq \frac{2ML\sqrt{d}}{c}\tau. 
\end{equation*}
Hence,
\begin{equation}\label{eqn:r2}
    \|\br_2(\tau)\|\leq \frac{2ML\sqrt{d}}{c}\tau.
\end{equation}
For $\br_2$, considering
\begin{equation*}
    (\nabla f(\bx(s)))^2-(\nabla f(\bx(\tau)))^2 = (\nabla f(\bx(s))-(\nabla f(\bx(\tau)))(\nabla f(\bx(s))+(\nabla f(\bx(\tau)))
\end{equation*}
and the upper bound for $\|\nabla f(\bx)\|$, similar to the estimation of $\br_2$ we have
\begin{equation}\label{eqn:r1}
    \|\br_1(\tau)\|\leq \frac{4M^2L\sqrt{d}}{c}\tau.
\end{equation}
By~\eqref{eqn:r2} and~\eqref{eqn:r1}, we have
\begin{align}
&\left\| \nabla f(\bx)\left(\sqrt{\frac{1}{(\nabla f(\bx))^2}}-\sqrt{\frac{1}{(\nabla f(\bx))^2+\br_1}}\right) \right\| \nonumber\\
&\leq M\left\|\frac{\sqrt{(\nabla f(\bx))^2+\br_1} - \sqrt{(\nabla f(\bx))^2}}{\sqrt{(\nabla f(\bx))^2((\nabla f(\bx))^2+\br_1)}}\right\| \nonumber \\
  & = M\left\|\frac{\br_1}{\sqrt{(\nabla f(\bx))^2((\nabla f(\bx))^2+\br_1)}(\sqrt{(\nabla f(\bx))^2+\br_1} + \sqrt{(\nabla f(\bx))^2})}\right\| \nonumber\\
  &\leq \frac{48M^3L\sqrt{d}}{c^4}\tau, \label{eqn:est1}
\end{align}
where the last line is derived by $|[\nabla f(\bx)]_i|\geq\frac{c}{2}$, and  $|[\br_1]_i|<\frac{c^2}{8}$ which comes from $\tau<\frac{c^3}{32M^2L\sqrt{d}}$. On the other hand, we have
\begin{align}\label{eqn:est2}
\left\| \sqrt{\frac{1}{(\nabla f(\bx))^2+\br_1}}(\nabla f(\bx)-(\nabla f(\bx)+\br_2)) \right\| \leq \frac{6ML\sqrt{d}}{c^2}\tau. 
\end{align}
Combining~\eqref{eqn:est1} and~\eqref{eqn:est2}, we have
\begin{align}
\|\dot{\bx}+\textrm{sign}(\nabla f(\bx))\| &\leq \left(\frac{48M^3L\sqrt{d}}{c^4}+\frac{6ML\sqrt{d}}{c^2}\right)\tau \nonumber\\
    &\leq \frac{54M^3L\sqrt{d}}{c^4}\tau,
\end{align}
where the second inequality comes from $c\leq M$. This completes the proof. 


\end{document}